\documentclass{article}

\usepackage[preprint]{neurips_2024}

\usepackage[utf8]{inputenc} % allow utf-8 input
\usepackage[T1]{fontenc}    % use 8-bit T1 fonts
\usepackage[colorlinks=true, citecolor=blue, urlcolor=red, linkcolor=magenta]{hyperref}

\usepackage{url}            % simple URL typesetting
\usepackage{booktabs}       % professional-quality tables
\usepackage{amsfonts}       % blackboard math symbols
\usepackage{nicefrac}       % compact symbols for 1/2, etc.
\usepackage{microtype}      % microtypography
\usepackage{xcolor}         % colors

\usepackage{amsfonts,amsmath,amssymb,amsthm,mathtools}
\usepackage{natbib}
\usepackage{booktabs}
\usepackage{placeins}
\usepackage{thmtools}

\usepackage{algorithm}
\usepackage{algpseudocode}

\newcommand{\W}{\mathbf{W}}

\title{Controlling changes to attention logits}

\author{%
Ben Anson\\
University of Bristol\\
\texttt{ben.anson@bristol.ac.uk}\\
\And
Laurence Aitchison\\
Mistral AI\\
\texttt{laurence.aitchison@gmail.com}\\
}

\begin{document}

\maketitle

\begin{abstract}
    Stability of neural network weights is critical when training transformer models. The query and key weights are particularly problematic, as they tend to grow large without any intervention. Applying normalization to queries and keys, known as `QK norm', fixes stability issues in practice, but is not always applicable. For example, QK norm is not compatible with Multi Latent Attention (MLA) because QK norm requires full materialization of queries and keys during inference, which is not done in MLA.
    In this paper we suggest that controlling the changes to logits is important for stability.
    We show that these changes are controllable by assigning parameter-dependent learning rates to the query and key weights. We find that our cheap intervention allows us to increase the base learning rate of the network, outperform other methods in the MLA setting, and achieve performance competitive with QK norm when using Multi-head Attention.
\end{abstract}

\section{Introduction}
Principled scaling of transformer models is crucial for efficiently training larger and more capable architectures. Maximal Update Parametrization ($\mu$P)~\citep{yang2022tensor} has emerged as a key technique in this area, enabling the transfer of optimal hyperparameters from smaller to larger models by carefully parameterizing the model. A core desideratum of $\mu$P is to control the magnitude of activations and their updates~\citep{dey2025don}, ensuring consistent training dynamics across different model widths. Regarding attention,~\cite{yang2022tensor} address attention logits blow-up as we increase model width by proposing a static attention scaling factor. While this static scaling helps control logit magnitude across different model widths, it does not address step-to-step changes in logits during longer training runs, which can become a major source of instability, particularly at high learning rates.

Attention logits are a well-known source of training instability~\citep{zhai2023stabilizing,team2025kimi}, prompting the development of interventions such as QK norm~\citep{henry2020query} and QK clip/MuonClip~\citep{team2025kimi} to ensure their stability. While QK norm is especially effective, it is ill-suited for Multi-head Latent attention (MLA)~\citep{liu2024deepseek}, as queries and keys are not fully materialized at inference-time for efficiency reasons~\citep{team2025kimi}. Other methods like QK clip require a bespoke attention  to track maximum attention logits, which can complicate integration into existing codebases.
Thus there is a gap for an easy-to-implement intervention that improves training stability, but is more widely applicable than QK norm.

We approach this gap with a $\mu$P-inspired desideratum for the attention logits. Instead of constraining logit magnitudes, like QK clip, MuonClip, and QK norm, we seek to control the {\it change\/} in logits as we train. This preserves expressivity, while reining-in instability. To validate our approach, we conduct pretraining experiments on a 1B parameter model. Our results demonstrate that our method is as stable as QK norm, particularly at high base learning rates. While not quite reaching the same peak performance as QK norm in the standard Multi-Head Attention (MHA) setting, our method is computationally cheaper and is applicable to MLA. When used with MLA our method enables higher base learning rates, outperforms QK clip, highlighting its practical value for training modern, efficient transformer architectures.

In summary, our contributions are as follows,
\begin{itemize}
\item We propose that change in logits is an important metric to account for when training attention modules. 
\item We show that we can control the change in logits by modulating the learning rate of query weight based on the norms of corresponding key weight, and vice versa.
\item We demonstrate that this change to the learning rate leads to better validation loss performance than alternatives when training with MLA.
\end{itemize}
\begin{figure}[t]
\centering
\includegraphics[scale=0.9]{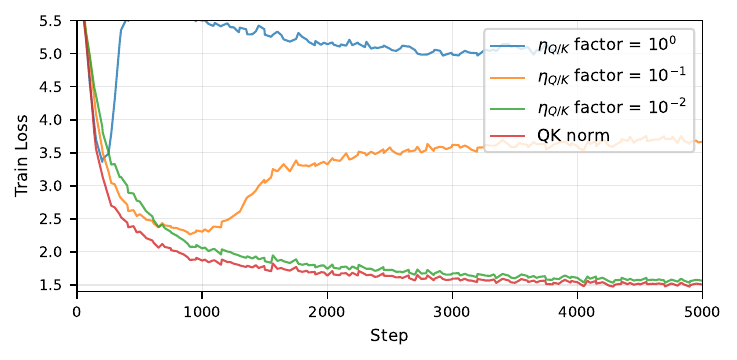}
\caption{\textbf{Learning rate of query/key matrices is a critical factor for transformer pretraining stability. } Here, 4 models are trained with a large base learning rate of $\eta = \mathtt{3e-2}$ for each parameter. Decreasing the learning rates of query and key weights alone (by a factor of $\eta_{Q/K}$), fully stabilizes pretraining. QK norm is shown to illustrate a stable baseline.}\label{fig:val_loss_large_high}
\end{figure}
\FloatBarrier

\section{Related Work}
Many have encountered instability issues when training transformers, and it has been studied extensively~\citep{liu2020understanding,dehghani2023scaling,henry2020query,team2025kimi,wortsman2023small,qi2023lipsformer,kim2025peri,zhai2023stabilizing,takase2023spike,rybakov2024methods}. Below we discuss the past literature relevant to our work.

\textbf{Controlling attention logits. } Training instabilities are often encountered in the attention layer itself. Attention logits may become large~\citep{team2025kimi}, potentially inducing collapse in attention entropy~\citep{zhai2023stabilizing}, where attention distributions become highly concentrated. QK normalization~\citep{henry2020query}, which applies normalization to query and key activations, has emerged as a simple and effective remedy, preventing large logits~\citep{dehghani2023scaling} and allowing larger learning rates~\citep{wortsman2023small}. Similar methods such as logit soft-capping apply normalization to logits directly~\citep{bello2016neural,team2024gemma}. Other methods normalize the weights rather than activations: $\sigma$Reparam~\citep{zhai2023stabilizing} parameterizes weights into a matrix and a scalar component, with the matrix having a maximum singular value of 1,  and a  weights rather than the activations; QK clip~\citep{team2025kimi} controls attention logits by clipping weights whenever the logits grow beyond a certain threshold.

\textbf{Parameter-specific learning rates. } While it is common to share the same learning rate across all parameters in a neural network, parameter-specific learning rates have been extensively examined~\citep{milsom2025functionspace,you2017large,liu2019variance,xu2019learning,wang2025sharpness,bernstein2020distance,qi2025taming,yang2023spectral}. Proposals often include adjusting the learning rate of a parameter according to the norm of step/gradient~\citep{yang2023spectral,liu2019variance}, as well as the parameter itself~\citep{qi2025taming}, such as LARS, LAMB, and Fromage~\citep{,bernstein2020distance,you2017large,you2019large}.

Our work selects parameter-specific learning rates that control changes to attention logits. However, by considering attention logits as a whole, our parameter-specific learning rates are `inter-parameter', unlike other methods, such as LARS, which consider each parameter tensor independently. Our method is also inspired by $\mu$P~\citep{yang2022tensor,dey2025don}; in $\mu$P, one of the desiderata is that as we make changes to our parameters in a network, the residual stream should correspondingly change in a controlled, `order 1-like' manner. Our work extends this notion to logits.
\section{Methods}\label{sec:methods}
Unlike other transformer modules, attention has quadratic structure. In particular, the attention logits are given by,
\begin{align}
\mathbf{L} = \frac{\mathbf{Q}\mathbf{K}^T}{\sqrt{d}} = \frac{\mathbf{X}\mathbf{W}_Q\mathbf{W}_K^T\mathbf{X}^T}{\sqrt{d}}.\label{eq:basic_attn}
\end{align}
Inspired by $\mu$P~\citep{yang2022tensor, dey2025don}, which (among other things) attempts to keep changes to {\it activations\/} roughly constant, we are interested in keeping the changes to {\it attention logits\/}, $\Delta \mathbf{L}$, under control. By a first-order analysis, we see that if the queries are large, then perturbations due to the keys will be amplified, and vice versa:
\begin{align}
\Delta\mathbf{L} = \frac{(\mathbf{Q } + \Delta\mathbf{Q})(\mathbf{K} + \Delta\mathbf{K})^T - \mathbf{Q}\mathbf{K}^T}{\sqrt{d}} 
\approx \frac{\mathbf{Q }(\Delta\mathbf{K})^T + (\Delta\mathbf{Q})\mathbf{K}^T}{\sqrt{d}}.
\end{align}
The main tool we have for controlling changes is the learning rate. Thus we propose to set the learning rates $\eta_{Q},\,\eta_{K}$ (for $\W_Q,$ and $\W_K$ respectively) such that $\mathbf{Q }(\Delta\mathbf{K})^T$ and $(\Delta\mathbf{Q})\mathbf{K}^T$ are both `order 1'.
We formalize this notion in Lemma~\ref{lemma:avgexpchange}.
\begin{restatable}{lemma}{avgexpchange}\label{lemma:avgexpchange}
Let $\mathbf{W}_Q,\mathbf{W}_K\in\mathbb{R}^{d_\text{model}\times d_\text{head}}$ be weight matrices corresponding to a particular attention head, and consider
the worst-case change in logits, for unit normed input,
\[
\max_{\|x\|_2=\|y\|_2=1}\lvert \Delta \ell\rvert:= \max_{\|x\|_2=\|y\|_2=1}
\lvert x^\top(\mathbf{W}+\Delta \mathbf{W})y-x^\top\mathbf{W}y\rvert,
\]
where \(\mathbf{W} = d_\text{head}^{-1/2}\mathbf{W}_Q^\top \mathbf{W}_K\). Suppose that the steps for $\mathbf{W}_Q$ and $\mathbf{W}_K$ are given by $\Delta\mathbf{W}_{Q/K} = -\eta_{Q/K} \mathbf{G}_{Q/K}$, where $\|\mathbf{G}_{Q/K}\|\leq D$ for some constant $D$ (which is the case for Adam and Muon).
If there is a constant $c$ such that $0 < c \leq \|\mathbf{W}_Q\|,\,\|\mathbf{W}_K\|$, and the learning rates satisfy $\eta_Q\propto \|\mathbf{W}_K\|^{-1}$, and $\eta_K\propto \|\mathbf{W}_Q\|^{-1}$, then the worst-case change in logits is bounded above independently of the weight size.
\end{restatable}
For Lemma~\ref{lemma:avgexpchange} to apply, we need a constant $c$ such that $c\leq \|\mathbf{W}_{Q/K}\|$, but this is not unreasonable in practice. The Lemma does not specify a norm because all norms are equivalent, though in practice, we do need to pick a norm for an implementation. The most natural norm for restricting the maximum change to the logits is perhaps the spectral norm. In a preliminary experiment, we compared the performance of both Frobenius and spectral norm, with results shown in Figure~\ref{fig:spec_frob_comparison}. The benefits of using the spectral norm are very small, thus we opted to use the Frobenius norm for further experiments in Section~\ref{sec:experiments}. 

Following the Lemma we set,
\begin{align}\label{eq:quack_eqns}
\eta_Q \propto \|\W_K\|^{-1},\;\;\eta_K \propto \|\W_Q\|^{-1}.
\end{align}
In practice we treat the constant of proportionality in Eq.~\eqref{eq:quack_eqns} as a hyperparameter: at initialization, we  set the learning rate for each query and key weight to be equal to $\tau\eta$ and we tune $\tau$. Thus $\tau$ acts as a relative initial learning rate (relative to $\eta$, the base learning rate).

The above methodology applies to both the single- and multi-head (MHA) setting. In MHA, each head has its own query and key weight, so we apply~Eq.~\eqref{eq:quack_eqns} to each head separately. 
We summarize the resulting method in Algorithm~\ref{alg:mha_quack}.
\begin{algorithm}[t]
\caption{QuacK (MHA)}
\label{alg:mha_quack}
\begin{algorithmic}
\Require Hyperparameter $\tau$, base learning rate $\eta$
\Statex Make the following additions to the transformer training script:
\Statex
\Statex \textcolor{gray}{\#} \textcolor{blue}{At initialization.} \textcolor{gray}{ Calculate initial norms for query/key weights for all heads}
\ForAll{layers $\ell$}
  \ForAll{heads $h$}
    \State $\mathbf{W}_Q^{\ell,h}.\mathtt{init\_norm} \gets \| \mathbf{W}_Q^{\ell,h} \|$
    \State $\mathbf{W}_K^{\ell,h}.\mathtt{init\_norm} \gets \| \mathbf{W}_K^{\ell,h} \|$
  \EndFor
\EndFor
\Statex
\Statex \textcolor{gray}{\#} \textcolor{blue}{During training. }\textcolor{gray}{ Prior to each optimization step, adjust learning rates}
\ForAll{layers $\ell$}
  \ForAll{heads $h$}
    \State $\mathbf{W}_Q^{\ell,h}.\mathtt{lr} \gets \tau\,\eta \cdot 
      \dfrac{\mathbf{W}_K^{\ell,h}.\mathtt{init\_norm}}{\| \mathbf{W}_K^{\ell,h} \|}$
    \State $\mathbf{W}_K^{\ell,h}.\mathtt{lr} \gets \tau\,\eta \cdot 
      \dfrac{\mathbf{W}_Q^{\ell,h}.\mathtt{init\_norm}}{\| \mathbf{W}_Q^{\ell,h} \|}$
  \EndFor
\EndFor
\end{algorithmic}
\end{algorithm}

We extend to MLA using a similar approach in Appendix~\ref{appendix:mla}. A notable difference between MHA and MLA is that there are several more parameter matrices to consider; bounding the change in logits requires us to adjust the learning rate of each of these parameters. We detail exactly how to set the learning rates for MLA in Algorithm~\ref{alg:mla_mha}.
\section{Experiments}\label{sec:experiments}
\begin{figure}[t]
\centering
\includegraphics[]{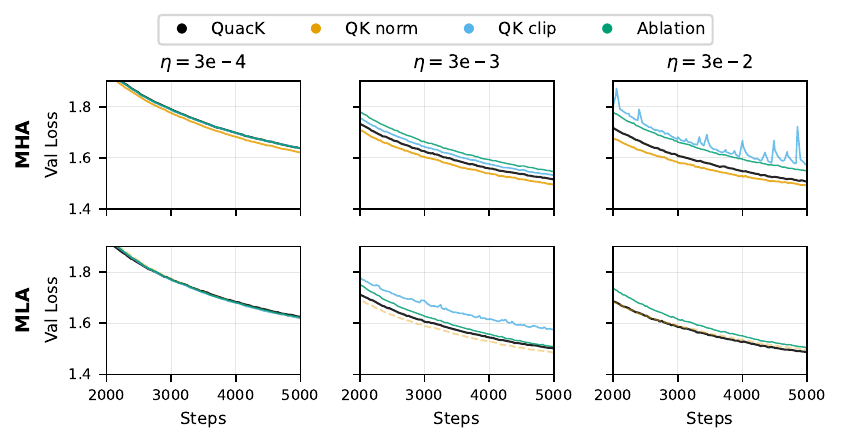}
\caption{Validation losses when training each method with $\mathtt{attn}\in\{\mathtt{MHA},\,\mathtt{MLA}\}$, and learning rates, $\eta \in\{\mathtt{3e-4}, \mathtt{3e-3}, \mathtt{3e-2}\}$. QK clip is unstable at high learning rates (it is omitted from the bottom right plot due to loss $\gg 2$). QK norm is overall the most performant, but it is not appropriate for use with MLA at inference-time for efficiency reasons (illustrated via dashed yellow line in the MLA row). QuacK is a sensible alternative, as it is stable in the high LR setting, performant, and applicable in the MLA setting. }\label{fig:plot_all}
\end{figure}
\begin{figure}[t]
\centering
\includegraphics[scale=0.9]{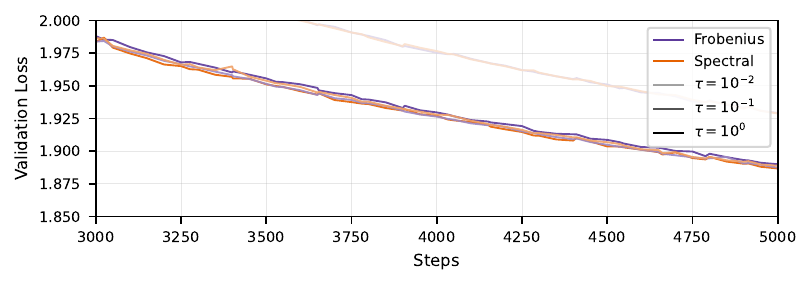}
\caption{Performance differences when applying Algorithm~\ref{alg:mha_quack} with different norms are small. We show validation losses when training a small model ($\sim 100$M parameters) with Algorithm~\ref{alg:mha_quack} to modulate the query and key weight learning rates. Different curves show results with different values of the hyperparameter $\tau$ and measuring the query and key weights with either Frobenius or spectral norm.
% The model uses MHA and the base learning rate is $\mathtt{3e-2}$.
}\label{fig:spec_frob_comparison}
\end{figure}
To evaluate our method, we trained $\sim 1$B models based on Qwen3~\citep{yang2025qwen3} with both MHA~\citep{vaswani2017attention} and MLA~\citep{liu2024deepseek}. All models used $d_\text{model} = 2048$, $d_\text{ff} = 4d_\text{model}$,  $n_\text{head} = 32$, $n_\text{layer} = 14$, and were trained using gradient accumulation at $2048$ context length with 96 sequences per batch (i.e. 196608 tokens per batch) for 5000 steps, using data from the Cosmopedia-V2 subset of SmolLM-corpus~\citep{benallal2024smollmcorpus}. We used the GPT-2~\citep{radford2019language} tokenizer, with vocab size 49152 and embedding/unembedding weight tying. The MHA models were trained with $d_\text{head} = 64$, while the MLA models were trained with $d_\text{head} = 128$ (with $d_\text{rope} =d_\text{nope} =  64$). MLA also used latent dimensions of $d_\text{cq} = 512$ for the queries and $d_\text{ckv} = 256$ for the keys and values. We trained using Muon~\citep{jordan2024muon}, with constant LR schedule and 500 warmup steps. All pretraining runs were completed on 4xGH200 nodes at bfloat16 precision.

Our experiments varied the attention method, $\mathtt{attn}\in\{\mathtt{MHA},\, \mathtt{MLA}\}$, the base learning rate $\eta \in \{\mathtt{3e-4},\, \mathtt{3e-3},\,\mathtt{3e-2}\}$, as well as the attention logit interventions:
\begin{itemize}

\item {\it QK norm\/} applies RMS norm (with learned scaling) to both queries and keys before applying the attention operation.
\item {\it QK clip\/} implements Algorithm 1 from~\cite{team2025kimi}, where after each optimization step we multiply/clip certain weights using either $\sqrt{\gamma}$ or $\gamma$. We set $\gamma = \min\{1, \tau_{\mathtt{QK\ clip}} / S_\text{max}^h\}$, where $\tau_\mathtt{QK\ clip}$ is a hyperparameter denoting a threshold for the maximum logit, and $S_\text{max}^h$ is the largest logit value seen by the $h$'th head since the last optimization step. For MHA we use $\mathbf{W}_{Q/K}^h \leftarrow \sqrt{\gamma}\;\mathbf{W}_{Q/K}^h$; for MLA we use $\mathbf{W}_\text{uq/uk} \leftarrow\sqrt{\gamma}\;\mathbf{W}_\text{uq/uk}$
and $\mathbf{W}_\text{qr} \leftarrow\gamma\;\mathbf{W}_\text{qr}$. We swept over $\tau_\mathtt{QK\ clip}\in\{30,100\}$, the values used in the original paper.
\item {\it Ablation\/} multiplies learning rates for query and key weights by a value $\tau$, which is swept over, $\tau\in\{10^{-2},10^{1}, 10^{0}, 10^{1}\}$. For MHA we set $\eta_Q^h = \eta_K^h = \tau \cdot \eta$. For MLA, we set $\eta_\text{uq}^h =
\eta_\text{dq} = 
\eta_\text{qr}^h = 
\eta_\text{uk}^h = 
\eta_\text{dkv} = 
\eta_\text{kr} = \tau \cdot \eta$.
\item {\it QuacK\/} also multiplies learning rates for query and key weights according to a value $\tau$, which is swept over, $\tau\in\{10^{-2},10^{1}, 10^{0}, 10^{1}\}$. For MHA we use Algorithm~\ref{alg:mha_quack}, and for MLA we use Algorithm~\ref{alg:mla_mha}. All weight norms are calculated using the Frobenius norm.
\end{itemize}
\textbf{Higher learning rates are better. } Figure~\ref{fig:plot_all}, left column, shows that at the low learning rate of $\eta = \mathtt{3e-4}$, all logit interventions perform similarly, but with QK norm performing marginally better.
The lack of variety in performance is likely due to the fact the learning rate is small enough that we don't encounter instabilities. However, performance is much improved by increasing the learning rate (column 2, 3).

\textbf{QuacK maintains stability and strong performance, enabling higher base learning rates. }
QK clip is insufficient to prevent instabilities at the highest base learning rate of $\eta=\mathtt{3e-2}$ (column 3, Figure~\ref{fig:plot_all}), and it underperforms, especially in the  MLA setting, when $\eta=\mathtt{3e-3}$ (column 2).
%At higher base learning rates ($\eta =\mathtt{3e-3}$, $\eta = \mathtt{3e-2}$), we see from Figure~\ref{fig:plot_all} that QK clip is insufficient to prevent instabilities.
The ablation, which sets the learning rates for query and key weights to smaller fixed values, is stable, but underperforms QuacK in both the MHA and MLA settings. QuacK has similar but slightly worse performance compared to QK norm in the MHA setting, but is the best performing method in the MLA setting (we include QK norm results in the MLA setting for comparison, but it is not viable at inference-time like the other methods).

\textbf{QuacK controls the maximum logit as well as change in logits. } During our training runs, we used a dedicated subset of our data to periodically record logit statistics. We confirm that QuacK controls the quantity we expect (average absolute change in logits) in Figure~\ref{fig:logit_stats}. We also see that either intervening with QK norm or controlling the learning rate appears sufficient to control the maximum logits. We include the `default' model (i.e. using a base learning rate of $\eta = \mathtt{3e-3}$ for all parameters), to illustrate behaviour of an unstable run (we killed this run prematurely, as it never got below 4.5 loss after 1.5k steps).
\begin{figure}[t]
\centering
\includegraphics[scale=0.9]{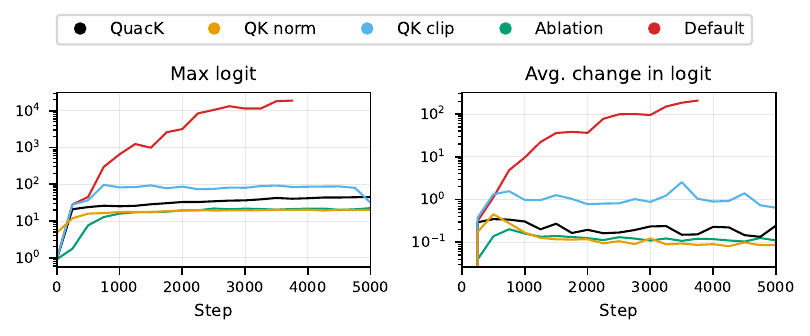}
\caption{Max logit (left) and average absolute change in logit throughout training (right) with a base learning rate of $\eta = \mathtt{3e-3}$. Here we show the middle head of the middle layer (head 16 and layer 8) while training with MLA.}\label{fig:logit_stats}
\end{figure}

\textbf{QuacK gives a performance boost over QK norm. }
QuacK yields a speedup over QK norm by removing two RMS norm computations per attention block; in practice we observed $\sim$10\% faster training. QuacK was also faster than QK clip, though this comparison is slightly unfair because we did not use an optimized QK clip implementation. QK clip can be efficient, but requires custom attention code to efficiently accumulate the maximum logit.
\section{Limitations}
Our experiments used a single model based on Qwen3~\citep{yang2025qwen3}. Due to compute constraints, we were unable to execute longer training runs and with larger models to demonstrate that our method is widely applicable. As such, results are limited by short training durations (5k steps) and a single dataset and model architecture.
\section{Conclusion}
In this work, we introduced a simple and principled approach for stabilizing attention training by controlling changes in attention logits rather than their magnitude. Our analysis showed that the expected logit change can be bounded through parameter-dependent learning rates for the query and key weights, inspired by µP-style scaling principles. Empirically, our method QuacK enables the use of higher base learning rates while maintaining stability comparable to QK norm and outperforming alternative methods such as QK clip, particularly in the Multi-Latent Attention (MLA) setting where QK norm is inappropriate.

Our results demonstrate that stability in attention can be achieved without introducing additional normalization layers or specialized kernels, making QuacK a practical drop-in improvement for transformer training.
\section{Acknowledgements}
We thank Edward Milsom for insightful discussions.
\clearpage
\bibliography{refs}

@article{ji2025towards,
  title={Towards Economical Inference: Enabling DeepSeek's Multi-Head Latent Attention in Any Transformer-based LLMs},
  author={Ji, Tao and Guo, Bin and Wu, Yuanbin and Guo, Qipeng and Shen, Lixing and Chen, Zhan and Qiu, Xipeng and Zhang, Qi and Gui, Tao},
  journal={arXiv preprint arXiv:2502.14837},
  year={2025}
}

@inproceedings{zhai2023stabilizing,
  title={Stabilizing transformer training by preventing attention entropy collapse},
  author={Zhai, Shuangfei and Likhomanenko, Tatiana and Littwin, Etai and Busbridge, Dan and Ramapuram, Jason and Zhang, Yizhe and Gu, Jiatao and Susskind, Joshua M},
  booktitle={International Conference on Machine Learning},
  pages={40770--40803},
  year={2023},
  organization={PMLR}
}

@article{team2025kimi,
  title={Kimi k2: Open agentic intelligence},
  author={Bai, Yifan and Bao, Yiping and Chen, Guanduo and Chen, Jiahao and Chen, Ningxin and Chen, Ruijue and Chen, Yanru and Chen, Yuankun and Chen, Yutian and others},
  journal={arXiv preprint arXiv:2507.20534},
  year={2025}
}

@inproceedings{dehghani2023scaling,
  title={Scaling vision transformers to 22 billion parameters},
  author={Dehghani, Mostafa and Djolonga, Josip and Mustafa, Basil and Padlewski, Piotr and Heek, Jonathan and Gilmer, Justin and Steiner, Andreas Peter and Caron, Mathilde and Geirhos, Robert and Alabdulmohsin, Ibrahim and others},
  booktitle={International conference on machine learning},
  pages={7480--7512},
  year={2023},
  organization={PMLR}
}

@article{wortsman2023small,
  title={Small-scale proxies for large-scale transformer training instabilities},
  author={Wortsman, Mitchell and Liu, Peter J and Xiao, Lechao and Everett, Katie and Alemi, Alex and Adlam, Ben and Co-Reyes, John D and Gur, Izzeddin and Kumar, Abhishek and Novak, Roman and others},
  journal={arXiv preprint arXiv:2309.14322},
  year={2023}
}

@article{liu2020understanding,
  title={Understanding the difficulty of training transformers},
  author={Liu, Liyuan and Liu, Xiaodong and Gao, Jianfeng and Chen, Weizhu and Han, Jiawei},
  journal={arXiv preprint arXiv:2004.08249},
  year={2020}
}

@article{henry2020query,
  title={Query-key normalization for transformers},
  author={Henry, Alex and Dachapally, Prudhvi Raj and Pawar, Shubham and Chen, Yuxuan},
  journal={arXiv preprint arXiv:2010.04245},
  year={2020}
}

@article{qi2023lipsformer,
  title={Lipsformer: Introducing lipschitz continuity to vision transformers},
  author={Qi, Xianbiao and Wang, Jianan and Chen, Yihao and Shi, Yukai and Zhang, Lei},
  journal={arXiv preprint arXiv:2304.09856},
  year={2023}
}

@article{kim2025peri,
  title={Peri-ln: Revisiting normalization layer in the transformer architecture},
  author={Kim, Jeonghoon and Lee, Byeongchan and Park, Cheonbok and Oh, Yeontaek and Kim, Beomjun and Yoo, Taehwan and Shin, Seongjin and Han, Dongyoon and Shin, Jinwoo and Yoo, Kang Min},
  journal={arXiv preprint arXiv:2502.02732},
  year={2025}
}

@article{takase2023spike,
  title={Spike no more: Stabilizing the pre-training of large language models},
  author={Takase, Sho and Kiyono, Shun and Kobayashi, Sosuke and Suzuki, Jun},
  journal={arXiv preprint arXiv:2312.16903},
  year={2023}
}

@article{rybakov2024methods,
  title={Methods of improving llm training stability},
  author={Rybakov, Oleg and Chrzanowski, Mike and Dykas, Peter and Xue, Jinze and Lanir, Ben},
  journal={arXiv preprint arXiv:2410.16682},
  year={2024}
}

@article{bello2016neural,
  title={Neural combinatorial optimization with reinforcement learning},
  author={Bello, Irwan and Pham, Hieu and Le, Quoc V and Norouzi, Mohammad and Bengio, Samy},
  journal={arXiv preprint arXiv:1611.09940},
  year={2016}
}

@article{team2024gemma,
  title={Gemma 2: Improving open language models at a practical size},
  author={Riviere, Morgane and Pathak, Shreya and Sessa, Pier Giuseppe and Hardin, Cassidy and Bhupatiraju, Surya and Hussenot, L{\'e}onard and Mesnard, Thomas and Shahriari, Bobak and Ram{\'e}, Alexandre and others},
  journal={arXiv preprint arXiv:2408.00118},
  year={2024}
}

@InProceedings{milsom2025functionspace,
  title = 	 {Function-Space Learning Rates},
  author =       {Milsom, Edward and Anson, Ben and Aitchison, Laurence},
  date = {2025},
  booktitle = 	 {Proceedings of the 42nd International Conference on Machine Learning},
  year = 	 {2025},
}

@article{you2017large,
  title={Large batch training of convolutional networks},
  author={You, Yang and Gitman, Igor and Ginsburg, Boris},
  journal={arXiv preprint arXiv:1708.03888},
  year={2017}
}

@article{liu2019variance,
  title={On the variance of the adaptive learning rate and beyond},
  author={Liu, Liyuan and Jiang, Haoming and He, Pengcheng and Chen, Weizhu and Liu, Xiaodong and Gao, Jianfeng and Han, Jiawei},
  journal={arXiv preprint arXiv:1908.03265},
  year={2019}
}

@article{xu2019learning,
  title={Learning an adaptive learning rate schedule},
  author={Xu, Zhen and Dai, Andrew M and Kemp, Jonas and Metz, Luke},
  journal={arXiv preprint arXiv:1909.09712},
  year={2019}
}

@article{wang2025sharpness,
  title={The sharpness disparity principle in transformers for accelerating language model pre-training},
  author={Wang, Jinbo and Wang, Mingze and Zhou, Zhanpeng and Yan, Junchi and Wu, Lei and others},
  journal={arXiv preprint arXiv:2502.19002},
  year={2025}
}

@article{bernstein2020distance,
  title={On the distance between two neural networks and the stability of learning},
  author={Bernstein, Jeremy and Vahdat, Arash and Yue, Yisong and Liu, Ming-Yu},
  journal={Advances in Neural Information Processing Systems},
  volume={33},
  pages={21370--21381},
  year={2020}
}

@article{qi2025taming,
  title={Taming Transformer Without Using Learning Rate Warmup},
  author={Qi, Xianbiao and He, Yelin and Ye, Jiaquan and Li, Chun-Guang and Zi, Bojia and Dai, Xili and Zou, Qin and Xiao, Rong},
  journal={arXiv preprint arXiv:2505.21910},
  year={2025}
}

@article{yang2023spectral,
  title={A spectral condition for feature learning},
  author={Yang, Greg and Simon, James B and Bernstein, Jeremy},
  journal={arXiv preprint arXiv:2310.17813},
  year={2023}
}

@article{you2019large,
  title={Large batch optimization for deep learning: Training bert in 76 minutes},
  author={You, Yang and Li, Jing and Reddi, Sashank and Hseu, Jonathan and Kumar, Sanjiv and Bhojanapalli, Srinadh and Song, Xiaodan and Demmel, James and Keutzer, Kurt and Hsieh, Cho-Jui},
  journal={arXiv preprint arXiv:1904.00962},
  year={2019}
}

@article{yang2022tensor,
  title={Tensor programs v: Tuning large neural networks via zero-shot hyperparameter transfer},
  author={Yang, Greg and Hu, Edward J and Babuschkin, Igor and Sidor, Szymon and Liu, Xiaodong and Farhi, David and Ryder, Nick and Pachocki, Jakub and Chen, Weizhu and Gao, Jianfeng},
  journal={arXiv preprint arXiv:2203.03466},
  year={2022}
}

@article{dey2025don,
  title={Don't be lazy: CompleteP enables compute-efficient deep transformers},
  author={Dey, Nolan and Zhang, Bin Claire and Noci, Lorenzo and Li, Mufan and Bordelon, Blake and Bergsma, Shane and Pehlevan, Cengiz and Hanin, Boris and Hestness, Joel},
  journal={arXiv preprint arXiv:2505.01618},
  year={2025}
}

@article{liu2024deepseek,
  title={Deepseek-v3 technical report},
  author={Liu, Aixin and Feng, Bei and Xue, Bing and Wang, Bingxuan and Wu, Bochao and Lu, Chengda and Zhao, Chenggang and Deng, Chengqi and Zhang, Chenyu and Ruan, Chong and others},
  journal={arXiv preprint arXiv:2412.19437},
  year={2024}
}

@article{su2021roformer,
  title={RoFormer: enhanced transformer with rotary position embedding. arXiv},
  author={Su, Jianlin and Lu, Yu and Pan, Shengfeng and Murtadha, Ahmed and Wen, Bo and Liu, Yunfeng},
  journal={arXiv preprint arXiv:2104.09864},
  year={2021}
}

@article{yang2025qwen3,
  title={Qwen3 technical report},
  author={Yang, An and Li, Anfeng and Yang, Baosong and Zhang, Beichen and Hui, Binyuan and Zheng, Bo and Yu, Bowen and Gao, Chang and Huang, Chengen and Lv, Chenxu and others},
  journal={arXiv preprint arXiv:2505.09388},
  year={2025}
}

@article{vaswani2017attention,
  title={Attention is all you need},
  author={Vaswani, Ashish and Shazeer, Noam and Parmar, Niki and Uszkoreit, Jakob and Jones, Llion and Gomez, Aidan N and Kaiser, Lukasz and Polosukhin, Illia},
  journal={Advances in neural information processing syStems},
  volume={30},
  year={2017}
}

@article{radford2019language,
  title={Language models are unsupervised multitask learners},
  author={Radford, Alec and Wu, Jeffrey and Child, Rewon and Luan, David and Amodei, Dario and Sutskever, Ilya and others},
  journal={OpenAI blog},
  volume={1},
  number={8},
  pages={9},
  year={2019}
}

@software{benallal2024smollmcorpus,
  author = {Ben Allal, Loubna and Lozhkov, Anton and Penedo, Guilherme and Wolf, Thomas and von Werra, Leandro},
  title = {SmolLM-Corpus},
  month = July,
  year = 2024,
  url = {https://huggingface.co/datasets/HuggingFaceTB/smollm-corpus}
}

@misc{jordan2024muon,
  author       = {Keller Jordan and Yuchen Jin and Vlado Boza and Jiacheng You and
                  Franz Cesista and Laker Newhouse and Jeremy Bernstein},
  title        = {Muon: An optimizer for hidden layers in neural networks},
  year         = {2024},
  url          = {https://kellerjordan.github.io/posts/muon/}
}
\bibliographystyle{icml2024}

\clearpage
\appendix
\section{Proof of Lemma~\ref{lemma:avgexpchange}}\label{appendix:proof_lemma}
\avgexpchange*
\begin{proof}
The change in logits is given by,
\begin{align}
d_\text{head}^{1/2}|\Delta \ell| &= \lvert(q + \Delta q)^T (k + \Delta k) - q^T k\rvert \nonumber \\
&= \lvert(\Delta q)^T k + q^T \Delta k + (\Delta q)^T \Delta k\rvert \nonumber \\
&\leq \lvert (\Delta q)^T k\rvert + \lvert q^T \Delta k\rvert + \|\Delta q\| \|\Delta k\| \label{eq:exp_mha_delta_ell}.
\intertext{where,}
q &= \mathbf{W}_Q x,\,k = \mathbf{W}_K y
\end{align}
The query and key perturbations are given by,
\begin{align}
\Delta q &= (\mathbf{W}_Q + \Delta \mathbf{W}_Q)x - \mathbf{W}_Q x = \Delta \mathbf{W}_Q x, \\
\Delta k &= (\mathbf{W}_K + \Delta \mathbf{W}_K)y - \mathbf{W}_K y = \Delta \mathbf{W}_K y.
\end{align}

We now bound the first order terms in Eq.~\eqref{eq:exp_mha_delta_ell}, assuming inputs are unit normed,
\begin{subequations}\label{eq:first_order_mha_bounds}
\begin{align}
\lvert (\Delta q)^T k\rvert &= \lvert (\Delta \mathbf{W}_Q x)^T (\mathbf{W}_K y) \rvert = \lvert x^T \Delta \mathbf{W}_Q^T \mathbf{W}_K y \rvert \nonumber \\
&\leq \|x\| \|\Delta \mathbf{W}_Q^T \mathbf{W}_K\| \|y\| \nonumber \\
&\leq \eta_Q D \|\mathbf{W}_K\|, \\[10pt]
\lvert q^T \Delta k\rvert &= \lvert (\mathbf{W}_Q x)^T (\Delta \mathbf{W}_K y) \rvert = \lvert x^T \mathbf{W}_Q^T \Delta \mathbf{W}_K y \rvert \nonumber \\
&\leq \|x\| \|\mathbf{W}_Q ^T\Delta \mathbf{W}_K\| \|y\| \nonumber \\
&\leq \eta_K D \|\mathbf{W}_Q\|.
\end{align}
\end{subequations}
For some constants $\tau_Q$, $\tau_K$, set,
\begin{subequations}\label{eq:mha_lrs}
\begin{align}
\eta_Q &= \tau_Q \|\mathbf{W}_K\|^{-1} \\
\eta_K &= \tau_K \|\mathbf{W}_Q\|^{-1}.
\end{align}
\end{subequations}
Substituting these into Eqs.~\eqref{eq:first_order_mha_bounds}, we obtain the bounds,
\begin{align}
\lvert (\Delta q)^T k\rvert &\leq (\tau_Q \|\mathbf{W}_K\|^{-1}) D \|\mathbf{W}_K\| = \tau_Q D, \\
\lvert q^T \Delta k\rvert &\leq (\tau_K \|\mathbf{W}_Q\|^{-1}) D \|\mathbf{W}_Q\| = \tau_K D.
\end{align}
Note that even if RoPE is applied, such that $q = R_x \mathbf{W}_Q x$, the bound remains identical as $\|R \mathbf{W}\| = \|\mathbf{W}\|$ (if the Frobenius or spectral is used).

Finally, we consider the quadratic term $\|\Delta q\|\|\Delta k\|$,
\begin{align}
\|\Delta q\| \|\Delta k\| &\leq \|\Delta \mathbf{W}_Q\| \|\Delta \mathbf{W}_K\| \nonumber \\
&= \frac{\tau_Q\tau_K\|\mathbf{G}_Q\|\|\mathbf{G}_K\|}{\|\mathbf{W}_Q\|\|\mathbf{W}_K\|} \nonumber \\
&\leq \frac{\tau_Q\tau_K D^2}{c^2}.
\end{align}
Thus the change in logits is bounded by a constant.
\end{proof}
\section{Extension to MLA}\label{appendix:mla}
In this section we motivate Algorithm~\ref{alg:mla_mha}, specifically the factors associated with each weight
\FloatBarrier
\begin{algorithm}[h]
\caption{QuacK (MLA)}
\label{alg:mla_mha}
\begin{algorithmic}
\Require Hyperparameter $\tau$, base learning rate $\eta$
\Statex Make the following additions to the transformer training script:
\Statex
\Function{$\mathtt{compute\_lr\_factors()}$}{}
  \ForAll{layers $\ell$}
    \ForAll{heads $h$}
      \State $\mathbf{W}_{\text{uq}}^{\ell,h}.\mathtt{factor} \gets (\|\mathbf{W}_{\text{dq}}^{\ell}\|\,\|\mathbf{W}_{\text{uk}}^{\ell,h}\|\,\|\mathbf{W}_{\text{dkv}}^{\ell}\|)^{-1}$
      \State $\mathbf{W}_{\text{uk}}^{\ell,h}.\mathtt{factor} \gets (\|\mathbf{W}_{\text{uq}}^{\ell,h}\|\,\|\mathbf{W}_{\text{dq}}^{\ell}\|\,\|\mathbf{W}_{\text{dkv}}^{\ell}\|)^{-1}$
      \State $\mathbf{W}_{\text{qr}}^{\ell,h}.\mathtt{factor} \gets (\|\mathbf{W}_{\text{dq}}^{\ell}\|\,\|\mathbf{W}_{\text{kr}}^{\ell}\|)^{-1}$
    \EndFor
    \State $\mathbf{W}_{\text{dq}}^{\ell}.\mathtt{factor} \gets 
      \min\!\Big\{
        (\max_h\|\mathbf{W}_{\text{uq}}^{\ell,h}\|\|\mathbf{W}_{\text{uk}}^{\ell,h}\|\|\mathbf{W}_{\text{dkv}}^{\ell}\|)^{-1},\,
        (\max_h\|\mathbf{W}_{\text{qr}}^{\ell,h}\|\|\mathbf{W}_{\text{kr}}^{\ell}\|)^{-1}
      \Big\}$
    \State $\mathbf{W}_{\text{dkv}}^{\ell}.\mathtt{factor} \gets (\max_h\|\mathbf{W}_{\text{uq}}^{\ell,h}\|\|\mathbf{W}_{\text{dq}}^{\ell}\|\|\mathbf{W}_{\text{uk}}^{\ell,h}\|)^{-1}$
    \State $\mathbf{W}_{\text{kr}}^{\ell}.\mathtt{factor} \gets (\max_h\|\mathbf{W}_{\text{qr}}^{\ell,h}\|\|\mathbf{W}_{\text{dq}}^{\ell}\|)^{-1}$
  \EndFor
\EndFunction

\Statex
\State $\{\mathtt{attention\_weights}\} \gets \{\mathbf{W}_{\text{uq}}^{\ell,h},\,\mathbf{W}_{\text{uk}}^{\ell,h},\,\mathbf{W}_{\text{qr}}^{\ell,h},\, \mathbf{W}_\text{dq}^\ell, \,\mathbf{W}_{\text{dkv}}^\ell, \, \mathbf{W}_\text{kr}^\ell\text{ for all layers } \ell \text{ for all heads } h   \}$

% \Statex \textcolor{gray}{\# At initialization: compute and attach initial learning rate factors}
\Statex
\Statex \textcolor{gray}{\#} \textcolor{blue}{At initialization.} \textcolor{gray}{ Compute initial learning rate factors for all attention weights}
\State $\mathtt{compute\_lr\_factors}()$
\ForAll{$\mathbf{W}$ in $\{\mathtt{attention\_weights}\}$}
\State $\mathbf{W}.\mathtt{init\_factor} \gets \mathbf{W}.\mathtt{factor}$
\EndFor

% \Statex \textcolor{gray}{\# During training: before each optimization step, adjust learning rates}
\Statex
\Statex \textcolor{gray}{\#} \textcolor{blue}{During training. }\textcolor{gray}{ Prior each optimization step, adjust learning rates}
\State $\mathtt{compute\_lr\_factors}()$
\ForAll{$\mathbf{W}$ in $\{\mathtt{attention\_weights}\}$}
\State $\mathbf{W}.\mathtt{lr} \gets \tau\,\eta \cdot 
  \dfrac{\mathbf{W}.\mathtt{factor}}{\mathbf{W}.\mathtt{init\_factor}}$
\EndFor
\end{algorithmic}
\end{algorithm}
We use a similar approach to Section~\ref{sec:methods} / Appendix~\ref{appendix:proof_lemma} when extending to MLA~\citep{ji2025towards,liu2024deepseek}. For now, assume the single-head setting. MLA tells us to calculate queries and keys as follows,
\begin{subequations}\label{eq:mla_def}
\begin{align}
% \mathbb{E}[|\Delta \ell|] &= \mathbb{E}_{x,y}[\mid(q + \Delta q)^T (k + \Delta k) - q^T k\mid ]\\
q &=  \text{Concat}(q_\text{nope}, q_\text{rope})\\
 k &=  \text{Concat}(k_\text{nope}, k_\text{rope})\\
q_\text{nope}  &= \W_\text{uq} \W_\text{dq} x\\
q_\text{rope}  &= R_x(\W_\text{qr} \W_\text{dq} x)\\
c_\text{kv} &= \W_\text{dkv} y\\
k_\text{nope}  &= \W_\text{uk}c_\text{kv} = \W_\text{uk}\W_\text{dkv} y\\
k_\text{rope}  &= R_y(\W_\text{kr} y).
\end{align}
\end{subequations}
Here, $x$ and $y$ are two token embeddings. The `down' matrices, $\mathbf{W}_\text{dq}$ and $\mathbf{W}_\text{dkv}$, project queries and keys/values respectively down to a lower dimensional latent space. This enables efficient caching of $c_\text{kv}$. The `up' matrices $\mathbf{W}_\text{uq}$, $\mathbf{W}_\text{uk}$ project these latents up to a higher dimensional space for attention calculations on each head. The $\mathbf{W}_\text{qr}$ and $\mathbf{W}_\text{kr}$ matrices are used to produce decoupled queries and keys for RoPE~\citep{su2021roformer} embeddings, with the position embedding applied via the rotation matrices $R_x$ and $R_y$.

The change in logits is given by,
\begin{align}
d_\text{head}^{1/2}|\Delta \ell| &= \lvert(q + \Delta q)^T (k + \Delta k) - q^T k\rvert = \lvert(\Delta q)^T k + q^T \Delta k +  (\Delta q)^T \Delta k\rvert, \\
\intertext{and we can bound the change,}
d_\text{head}^{1/2}|\Delta \ell|&\leq  \lvert (\Delta q)^T k\rvert  + \lvert q^T \Delta k\rvert + \|\Delta q\| \|\Delta k\|\nonumber\\
&\leq \lvert (\Delta q_\text{nope})^T k_\text{nope}\rvert +
\lvert (\Delta q_\text{rope})^T k_\text{rope}\rvert + 
\lvert q_\text{nope}^T \Delta k_\text{nope}\rvert + 
\lvert q_\text{rope}^T \Delta k_\text{rope}\rvert + \|\Delta q\| \|\Delta k\|\label{eq:exp_mla_delta_ell_first}.
\end{align}
Expanding further, for the queries, we have,
\begin{subequations}
\begin{align}
\Delta q_\text{nope} &= (\W_\text{uq} + \Delta\W_\text{uq})(\W_\text{dq} + \Delta\W_\text{dq})x - \W_\text{uq}\W_\text{dq}x\nonumber\\
&= \Delta\W_\text{uq}\W_\text{dq}x + \W_\text{uq}\Delta\W_\text{dq}x + \Delta \mathbf{W}_\text{uq}\Delta\mathbf{W}_\text{dq} x,\\
\Delta q_\text{rope} &= R_x[(\W_\text{qr} + \Delta\W_\text{qr})(\W_\text{dq} + \Delta\W_\text{dq})x - \W_\text{qr}\W_\text{dq}x]\nonumber\\
&= R_x[\Delta\W_\text{qr}\W_\text{dq}x + \W_\text{qr}\Delta\W_\text{dq}x + \Delta \mathbf{W}_\text{qr}\Delta \mathbf{W}_\text{dq} x],
\end{align}
\end{subequations}
and for the keys,
\begin{subequations}
\begin{align}
\Delta k_\text{nope} &= (\W_\text{uk} + \Delta\W_\text{uk})(\W_\text{dkv} + \Delta\W_\text{dkv})y - \W_\text{uk}\W_\text{dkv}y\nonumber\\
&= \Delta\W_\text{uk}\W_\text{dkv}y + \W_\text{uk}\Delta\W_\text{dkv}y + \Delta \mathbf{W}_\text{uk}\Delta \mathbf{W}_\text{dkv}y\\
\Delta k_\text{rope} &= R_y[(\W_\text{kr} + \Delta\W_\text{kr})y - \W_\text{kr}y] = R_y\Delta\W_\text{kr}y.
\end{align}
\end{subequations}
We now use these expressions, and the expressions for $q_\text{nope},\,k_\text{nope}, q_\text{rope}, \,k_\text{rope}$, to bound each of the terms in Eq.~\eqref{eq:exp_mla_delta_ell_first}.
We will make some assumptions (similar to Lemma~\ref{lemma:avgexpchange},
\begin{itemize}
\item the inputs $x$ and $y$ are unit normed;
\item we use the Frobenius norm;
\item conditioned gradients are bounded by a constant, i.e. $\Delta \mathbf{W}_\mathtt{x} = -\eta_\mathtt{x}\mathbf{G}_\mathtt{x}$ where $\|\mathbf{G}_\mathtt{x}\|\leq D$ (valid for Muon and Adam);
\item the weight norms are lower bounded by a constant $c$.
\end{itemize}
We consider the first order terms. We have,
\begin{subequations}\label{eq:first_order_mla_bounds}
\begin{align}
\lvert \Delta q_\text{nope}^T k_\text{nope}\rvert 
&= \lvert (\Delta\W_\text{uq}\W_\text{dq}x + \W_\text{uq}\Delta\W_\text{dq}x + \Delta \mathbf{W}_\text{uq}\Delta\mathbf{W}_\text{dq} x)^T k_\text{nope} \rvert \nonumber \\
&\leq \|\Delta\W_\text{uq}\|\|\W_\text{dq}\|\|k_\text{nope}\| + \|\W_\text{uq}\|\|\Delta\W_\text{dq}\|\|k_\text{nope}\| + \|\Delta \mathbf{W}_\text{uq}\|\|\Delta\mathbf{W}_\text{dq}\| \|k_\text{nope}\|\nonumber\\
&\leq \eta_\text{uq}D\|\W_\text{dq}\|\|\W_\text{uk}\|\|\W_\text{dkv}\| + \eta_\text{dq}D\|\W_\text{uq}\|\|\W_\text{uk}\|\|\W_\text{dkv}\| + O(\eta_{\text{uq}}\eta_\text{dq} \|\mathbf{W}_\text{uk}\|\|\mathbf{W}_\text{dkv}\|)\\
\lvert (\Delta q_\text{rope})^T k_\text{rope}\rvert
&= \lvert (R_x[\Delta\W_\text{qr}\W_\text{dq}x + \W_\text{qr}\Delta\W_\text{dq}x + \Delta \mathbf{W}_\text{qr}\Delta \mathbf{W}_\text{dq} x])^T k_\text{rope} \rvert \nonumber \\
&\leq (\|\Delta\W_\text{qr}\|\|\W_\text{dq}\| + \|\W_\text{qr}\|\|\Delta\W_\text{dq}\| + \|\Delta \mathbf{W}_\text{qr}\|\|\Delta \mathbf{W}_\text{dq}\|)\|k_\text{rope}\| \nonumber \\
&\leq \eta_\text{qr}D\|\W_\text{dq}\|\|\W_\text{kr}\| + \eta_\text{dq}D\|\W_\text{qr}\|\|\W_\text{kr}\| + O(\eta_\text{dq}\eta_\text{qr} \|\mathbf{W}_\text{kr}\|),\\[10pt]
\lvert q_\text{nope}^T \Delta k_\text{nope}\rvert
&= \lvert q_\text{nope}^T (\Delta\W_\text{uk}\W_\text{dkv}y + \W_\text{uk}\Delta\W_\text{dkv}y + \Delta \mathbf{W}_\text{uk}\Delta \mathbf{W}_\text{dkv}y) \rvert \nonumber \\
&\leq \|q_\text{nope}\|\|\Delta\W_\text{uk}\|\|\W_\text{dkv}\| + \|q_\text{nope}\|\|\W_\text{uk}\|\|\Delta\W_\text{dkv}\| + \|q_\text{nope}\|\|\Delta \mathbf{W}_\text{uk}\|\|\Delta \mathbf{W}_\text{dkv}\| \nonumber \\
&\leq \eta_\text{uk}D\|\W_\text{uq}\|\|\W_\text{dq}\|\|\W_\text{dkv}\| + \eta_\text{dkv}D\|\W_\text{uq}\|\|\W_\text{dq}\|\|\W_\text{uk}\| + O(\eta_\text{uk}\eta_\text{dkv} \|\mathbf{W}_\text{uq}\|\|\mathbf{W}_\text{dq}\|),\\[10pt]
\lvert q_\text{rope}^T \Delta k_\text{rope}\rvert
&= \lvert q_\text{rope}^T (R_y\Delta\W_\text{kr}y) \rvert
\leq \|q_\text{rope}\| \|\Delta\W_\text{kr}\|\leq \eta_\text{kr}D\|\W_\text{qr}\|\|\W_\text{dq}\|.
\end{align}
\end{subequations}
We used the fact that for rotation matrices $R$, $\|\mathbf{W}R\| =\|R\mathbf{W}\| = \|\mathbf{W}\|$.

Ultimately, Eqs.~\eqref{eq:first_order_mla_bounds} suggest to set the learning rates for each attention weight parameter as,
\begin{subequations}\label{eq:mla_single_head_lrs}
\begin{align}
\eta_\text{uq} &= \tau(\|\W_\text{dq}\|\|\W_\text{uk}\|\|\W_\text{dkv}\|)^{-1} \\
\eta_\text{dq} &= \tau \,\min\left\{(\|\W_\text{uq}\|\|\W_\text{uk}\| \|\W_\text{dkv}\|)^{-1}, (\|\W_\text{qr}\|\|\W_\text{kr}\|)^{-1}\right\}\\
\eta_\text{qr} &= 	\tau(\|\W_\text{dq}\|\|\W_\text{kr}\|)^{-1}\\
\eta_\text{uk} &= 	\tau(\|\W_\text{uq}\|\|\W_\text{dq}\|\|\W_\text{dkv}\|)^{-1}\\
\eta_\text{dkv} &= \tau(\|\W_\text{uq}\|\|\W_\text{dq}\|\|\W_\text{uk}\|)^{-1}\\
\eta_\text{kr} &= 	\tau(\|\W_\text{qr}\|\|\W_\text{dq}\|)^{-1}.
\end{align}
\end{subequations}
We then substitute these learning rates into Eqs.~\eqref{eq:first_order_mla_bounds}, to see that the bounds are given by,
\begin{subequations}
\begin{align}
\lvert (\Delta q_\text{nope})^T k_\text{nope}\rvert &\leq \tau D + \tau D + O\left(\frac{\tau^2 \|\mathbf{W}_\text{uk}\|\|\mathbf{W}_\text{dkv}\|}{\|\W_\text{dq}\|\|\W_\text{uk}\|\|\W_\text{dkv}\| \cdot \|\W_\text{uq}\|\|\W_\text{uk}\| \|\W_\text{dkv}\|}\right) \nonumber\\
&= 2\tau D + O\left(\frac{\tau^2}{\|\W_\text{dq}\|\|\W_\text{uq}\|\|\W_\text{uk}\|\|\W_\text{dkv}\|}\right),\\
\lvert (\Delta q_\text{rope})^T k_\text{rope}\rvert
&\leq \tau D + \tau D + O\left(\frac{\tau^2 \|\mathbf{W}_\text{kr}\|}{\|\W_\text{dq}\|\|\W_\text{kr}\| \cdot \|\W_\text{qr}\|\|\W_\text{kr}\|}\right) \nonumber\\
&= 2\tau D + O\left(\frac{\tau^2}{\|\W_\text{dq}\|\|\W_\text{qr}\|\|\W_\text{kr}\|}\right),\\
\lvert q_\text{nope}^T \Delta k_\text{nope}\rvert 
&\leq \tau D + \tau D + O\left(\frac{\tau^2 \|\mathbf{W}_\text{uq}\|\|\mathbf{W}_\text{dq}\|}{\|\W_\text{uq}\|\|\W_\text{dq}\|\|\W_\text{dkv}\| \cdot \|\W_\text{uq}\|\|\W_\text{dq}\|\|\W_\text{uk}\|}\right) \nonumber\\
&= 2\tau D + O\left(\frac{\tau^2}{\|\W_\text{uq}\|\|\W_\text{dq}\|\|\W_\text{uk}\|\|\W_\text{dkv}\|}\right),\\
\lvert q_\text{rope}^T \Delta k_\text{rope}\rvert &\leq \tau D.
\end{align}
\end{subequations}
It is reasonable to assume in practice that the weights are not arbitrarily small (i.e. their norm is lower bounded), and thus that these terms are bounded by a constant.

The only remaining term to bound in Eq.~\eqref{eq:exp_mla_delta_ell_first} is the quadratic term, $\|\Delta q\|\|\Delta k\|$. We can show that this is bounded by showing that the individual parts are bounded,
\begin{subequations}
\begin{align}
\|\Delta q_\text{nope}\| &\leq \eta_\text{uq}D\|\W_\text{dq}\| + \eta_\text{dq}D\|\W_\text{uq}\| + \eta_\text{uq}\eta_\text{dq}D^2 \nonumber\\
&\leq \frac{2\tau D}{\|\W_\text{uk}\|\|\W_\text{dkv}\|} + \eta_\text{uq}\eta_\text{dq}D^2,\\[10pt]
\|\Delta q_\text{rope}\| &\leq \eta_\text{qr}D\|\W_\text{dq}\| + \eta_\text{dq}D\|\W_\text{qr}\| + \eta_\text{qr}\eta_\text{dq}D^2 \nonumber\\
&\leq \frac{2\tau D}{\|\W_\text{kr}\|} + \eta_\text{qr}\eta_\text{dq}D^2,\\
\|\Delta k_\text{nope}\| &\leq \eta_\text{uk}D\|\W_\text{dkv}\| + \eta_\text{dkv}D\|\W_\text{uk}\| + \eta_\text{uk}\eta_\text{dkv}D^2 \nonumber\\
&\leq \frac{2\tau D}{\|\W_\text{uq}\|\|\W_\text{dq}\|} + \eta_\text{uk}\eta_\text{dkv}D^2,\\[10pt]
\|\Delta k_\text{rope}\| &\leq \eta_\text{kr}D \leq \frac{\tau D}{\|\W_\text{qr}\|\|\W_\text{dq}\|}.
\end{align}
\end{subequations}
To extend further to the multi-head setting, we add head indices to the necessary matrices, $\mathbf{W}^h_\text{uq},\, \mathbf{W}^h_\text{uk},$ and $\mathbf{W}^h_\text{qr}$, and their corresponding learning rates, $\eta^h_\text{uq},\,\eta^h_\text{uk},\,\eta^h_\text{qr}$. The key used for RoPE, $k_\text{rope}$ is shared between all heads, therefore $\mathbf{W}_\text{kr}$ surprisingly does not have a head index. The down matrices project to a latent space, so also do not have head indices.
Plugging these into Eqs.~\eqref{eq:mla_single_head_lrs} we have,
\begin{subequations}\label{eq:mla_multihead}
\begin{align}
\eta_\text{uq}^h &= \tau(\|\W_\text{dq}\|\|\W_\text{uk}^h\|\|\W_\text{dkv}\|)^{-1}\\
\eta_\text{dq} &= \tau \, \min\left\{\left(\max_h\|\W_\text{uq}^h\|\|\W_\text{uk}^h\| \|\W_\text{dkv}\|\right)^{-1},  \left(\max_h\|\W_\text{qr}^h\|\|\W_\text{kr}\|\right )^{-1}\right\}\\
\eta_\text{qr}^h &= 	\tau(\|\W_\text{dq}\|\|\W_\text{kr}\|)^{-1}\\
\eta_\text{uk}^h &= 	\tau(\|\W_\text{uq}^h\|\|\W_\text{dq}\|\|\W_\text{dkv}\|)^{-1}\\
\eta_\text{dkv} &= \tau(\max_h\|\W_\text{uq}^h\|\|\W_\text{dq}\|\|\W_\text{uk}^h\|)^{-1}\\
\eta_\text{kr} &= 	\tau(\max_h\|\W_\text{qr}^h\|\|\W_\text{dq}\|)^{-1}.
\end{align}
\end{subequations}
The use of $\max_h (\cdot)$ comes from the requirement that we want logit changes to be bounded for all heads.

%%%%%%%%%%%%%%%%%%%%%%%%%%%%%%%%%%%%%%%%%%%%%%%%%%%%%%%%%%%%

%%%%%%%%%%%%%%%%%%%%%%%%%%%%%%%%%%%%%%%%%%%%%%%%%%%%%%%%%%%%

\end{document}